\documentclass[a4paper,UKenglish]{oasics}

\usepackage{microtype}

\usepackage{latexsym} 
\usepackage{amsthm}
\usepackage{url}
\usepackage{multicol}
\usepackage[super]{nth}
\usepackage{thmtools}
\usepackage{thm-restate}

\usepackage[title]{appendix}

\usepackage{graphicx}
\newcommand{\minus}{\scalebox{0.65}[1.0]{$-$}}
\newcommand{\minusminus}{\minus 1/\minus 1 }
\newcommand{\tm}{\textit{Magic} Turing machine }

\usepackage{verbatim} 

\theoremstyle{plain}
\newtheorem{thm}{Theorem}

\theoremstyle{definition}
\newtheorem{df}{Definition}
\newtheorem{cor}{Corollary}

\theoremstyle{remark}
\newtheorem{conj}{Conjecture}

\newtheorem{quest}[thm]{Question}

\newenvironment{usethmcounterof}[1]{%
  \renewcommand\thethm{\ref{#1}}\thm}{\endthm\addtocounter{thm}{-1}}
 
\begin{document}
\title{Magic: the Gathering is as Hard as Arithmetic}

\author[1]{Stella Biderman}
\affil[1]{The Georgia Institute of Technology\\
  Atlanta, United States of America\\
  \texttt{stellabiderman@gatech.edu}}
\authorrunning{S. Biderman}

\Copyright{Stella Biderman}

\subjclass{Theory of computation $\rightarrow$ Algorithmic game theory}

\serieslogo{}
\volumeinfo
  {Billy Editor and Bill Editors}
  {2}
  {Conference/workshop/symposium title on which this volume is based on}
  {1}
  {1}
  {1}
\EventShortName{}
\DOI{10.4230/OASIcs.xxx.yyy.p}

\maketitle

\begin{abstract}
\textit{Magic: the Gathering} is a popular and famously complicated card game about magical combat. Recently, several authors including Chatterjee and Ibsen-Jensen (2016) and  Churchill, Biderman, and Herrick (2019) have investigated the computational complexity of playing \textit{Magic} optimally. In this paper we show that the ``mate-in-$n$'' problem for \textit{Magic} is $\Delta^0_n$-hard and that optimal play in two-player Magic is non-arithmetic in general. These results apply to how real Magic is played, can be achieved using standard-size tournament legal decks, and do not rely on stochasticity or hidden information. Our paper builds upon the construction that Churchill, Biderman, and Herrick (2019) used to show that this problem was at least as hard as the halting problem.
\end{abstract}

\section{Introduction}
\subsection{Background}
\textit{Magic: the Gathering} (also known as \textit{Magic}) is a popular trading card game owned by Wizards of the Coast. Due to it's highly complex game play and popularity with mathematically-minded players, there has been a significant amount of research on the computational complexity of the game \cite{ci:blocking,c:old-tm,c:forum,cbh:mtg-tm,e:thesis}.

Although there are many questions about games that are investigated in algorithmic game theory, the central and most important one is \textit{how hard is it to play a game optimally?} The investigation of this question for \textit{Magic} was begun by Churchill, Biderman, and Herrick (2019) \cite{cbh:mtg-tm}, who proved the following theorem:

\begin{restatable}{thm}{zerop}
\label{thm:0player}
Determining the outcome of a game of \textit{Magic: the Gathering} in which all remaining moves are forced is undeciable.
\end{restatable}

This theorem establishes \textit{Magic} is unlike many other games in that there can be no remaining unforced moves for either player, and yet the outcome of the game can be difficult to determine. Based on this, we define an \textit{end game} in \textit{Magic} to be a game state in which there are no remaining decisions for either player to make and posit that the interesting version of the traditional ``mate-in-$n$'' question is to identify a sequences of $n$ moves that result in an end game, even if the game isn't formally over at that point in time. With this idea in mind, we prove the following theorems:

\begin{restatable}{thm}{mate}
\label{thm:mate}
The mate-in-$n$ problem for \textit{Magic: the Gathering} is $\Delta^0_n$-hard.
\end{restatable}
To prove this, we will use the \tm and the correspondence established by Post's Theorem to create games of \textit{Magic} where optimal play requires identifying the truth of arithmetic sentences with $n$ alternating quantifiers. As our notion of an \textit{end game} doesn't effect anything in the limit, this generalizes to give

\begin{restatable}{thm}{twop}
\label{thm:2player}
Determining if there exists a winning strategy in \textit{Magic: the Gathering} is non-arithmetic.
\end{restatable}

\subsection{Previous Work}

The computational complexity of checking the legality of a particular decision in \textit{Magic} (blocking) is investigated in \cite{ci:blocking} and is found to be coNP-complete. There have also been a number of papers investigating practical algorithms and artificial intelligence approaches to playing \textit{Magic} \cite{wcp:mcmc,e:thesis,wc:mcmc}. Esche (2018) \cite{e:thesis} briefly considers the theoretical computational complexity of \textit{Magic} and states an open problem that has a positive answer only if \textit{Magic} end-games are decidable. Churchill (2012) \cite{c:old-tm} began the investigation of the computational complexity of \textit{Magic: the Gathering} in general and Churchill, Biderman, and Herrick (2019) \cite{cbh:mtg-tm} prove that it is at least $\emptyset'$.

While Churchill, Biderman, and Herrick (2019) \cite{cbh:mtg-tm} is the only work we are aware of showing that a real-world two-player game is non-computable, other puzzle games have been shown to be non-computable. The first example was the indie game Braid by Hamilton (2014) \cite{h:braid-undecidable}. More recently, Demaine, Kopinsky, and Lynch (2020) \cite{dkl:recursed-not-recursive} show that the puzzle game Recursed is not recursive.

Additionally, there has been recent success at showing that multiplayer team games are undecidable. Coulombe and Lynch (2019) \cite{cl:smash} show that several video games including \textit{Super Smash Bros. Melee} and \textit{Mario Kart} are undecidable using the Constraint Logic framework \cite{dh:contraint-logic}. Demaine, Kopinsky, and Lynch (2020) \cite{dkl:recursed-not-recursive} consider these to be real-world undecidable games, though we respectfully disagree. The constructions require creating custom game boards that do not exist in the game as it was published. While the rules of these games allow for the \textit{potentiality} of non-computable optimal strategy, we feel that the game as it was published does not actual require non-computable optimal strategy.

\subsection{Our Contribution}
Our result extends the work of Churchill, Biderman, and Herrick \cite{cbh:mtg-tm} and proves that optimal play in \textit{Magic: the Gathering} is at least as hard as $\emptyset^{(\omega)}$. Our result is the first result showing a real or realistic game has a $\Delta^0_n$-hard ``mate-in-$n$'' problem and the first result showing a real or realistic game can be non-arithmetic in general. Our review of the literature shows no sign that previous researchers had considered that real-world games could be harder than $\emptyset'$, and we hope that this example will encourage researchers to look for other games that are harder than $\emptyset'$.

\subsection{Overview}
The paper is structured as follows. In Section \ref{prelim} we provide background information relevant to this work, including previous work on \textit{Magic: the Gathering} Turing machines and some comments on "0-player" \textit{Magic}. In Section \ref{complexity} we extend this work to games with strategic decisions to prove \textbf{Theorem \ref{thm:mate}} and \textbf{Theorem \ref{thm:2player}}. In \autoref{sec:real-world} we discuss implementing our construction in a real-world environment. Finally, in Section \ref{futurework} we summarize our main points and identify avenues for future work.

Readers not familiar with \textit{Magic: the Gathering} can refer to Appendix \ref{app:rules} for an introduction to the rules of the game.

\section{Preliminaries}\label{prelim}

As discussed in Churchill, Biderman, and Herrick (2019) \cite{cbh:mtg-tm}, \textit{Magic: the Gathering} has an encoding problem. The rules allow players to choose numbers but doesn't specify how they are supposed to be encoded which can lead to non-computability that is incidental to the game's strategy, as determining if two different expressions represent the same number is non-computable in general. We will follow them by requiring players to specify numbers in binary notation to remove any ambiguity. Rule 107.1 of the \textit{Magic: the Gathering Comprehensive Rules} states that all game values in \textit{Magic} are integers\footnote{While some cards refer to operations that can produce non-integers such as division, they always specify how to round non-integer results.}, so binary notation can express any possible value unambiguously.

With this restriction, Churchill, Biderman, and Herrick conjecture
\begin{conj}\label{tcomp}
The function that takes a board state and a legal move and returns the next board state in \textit{Magic: the Gathering} is computable.
\end{conj}
In this conjecture we say ``a legal move'' because it is also not obvious that checking to see if a move is legal is computable. This is a particularly thorny issue as \textit{Magic} has rules for how to proceed in the event of a player making an illegal move. This means that in some scenarios a move is ``illegal'' in the sense that the rules tell you to not do it, but not ``illegal'' in the sense that you're \textit{unable} to do it. At tournaments officials are allowed to disqualify players who, in their judgement, are deliberately making ``mistakes'' to gain advantage. While our construction doesn't involve these ``illegal but within the rules'' moves, it's possible that this is an issue future work has to grapple with. We, like Churchill, Biderman, and Herrick, leave handling checking the legality of a move to future work.

\subsection{``Zero-Player'' \textit{Magic: the Gathering}}

At no point in the operation of Churchill, Biderman, and Herrick's \tm does a player have the ability to make any moves or influence the game in any way, despite the fact that the game officially goes on for many more turns. This is an unusual property of \textit{Magic}, as in most games at least one player will have non-trivial strategic decisions for nearly the entirety of the game. This means that some \textit{Magic} board states can be thought of as simulations, akin to Conway's \textit{Game of Life}. Like \textit{Game of Life}, the evolution of that simulation is Turing complete in general. As a result, in order to study game-theoretic questions we need a broader notion of a game being ``finished'' than one would usually use. Even if a game has not formally ended, when all moves for both players are forced for the rest of the game the game is \textit{strategically finished}.

This property of \textit{Magic: the Gathering} motivates the following definition:
\begin{df} A game is in an \textit{end game state} if there are no unforced moves for either player remaining. In such a scenario, we call the game an \textit{end game}.
\end{df}
Using this definition, we can rephrase \textbf{Theorem \ref{thm:0player}} as
\begin{usethmcounterof}{thm:0player}[Churchill et al.]
Determining the winner of a \textit{Magic: the Gathering} end game is undecidable.
\end{usethmcounterof}
This definition also allows us to ask new questions about games, such as ``can we identify end games.'' For \textit{Magic}, the answer to this question is \textit{partially}, assuming Conjecture \ref{tcomp} and the following conjecture:
\begin{conj}\label{Elegal}
Determining if a player has a legal move in \textit{Magic: the Gathering} is equivalent to the halting problem.
\end{conj}
\autoref{tcomp} does not immediately give us an algorithm for determining the existence of a legal move because there are games of \textit{Magic} where players appear to have infinitely many meaningfully different possible moves. We conjecture that no such algorithm exists as we believe that this is an essential difficulty, and that there are board states that require checking the legality of arbitrarily many possible moves.
\begin{cor}
If \autoref{tcomp} and \autoref{Elegal} hold, then determining that a game of \textit{Magic: the Gathering} is in an end game state is equivalent to the halting problem.
\end{cor}

\section{Two-player \textit{Magic: the Gathering}}\label{complexity}
Now we will move on to two-player games of \textit{Magic: the Gathering}. The first question we wish to consider is the ``mate-in-$n$ problem'':
\begin{quest}[the Mate-in-$n$ Problem]
Given a game state and an integer $n$, is there a sequence of moves that forces a win for the first player in less than $n$ moves?
\end{quest}
As in the previous section, we are interested in this question in the context of end games, so we will consider a mate-in-$n$ to exist if $n$ moves are sufficient to arrive at an end game that results in the first player winning.
\mate*
\begin{proof}
We will show this by appealing to Post's Theorem in the standard way. We will modify the \tm to encode an arithmetic problem, $$(\exists x)(\exists y_1)(\forall y_2)\cdots (Qy_n) (P(x,y_1,\ldots,y_n) = 0)$$ where $Q$ is the appropriate quantifier on the tape, halting if there is a solution and not halting if there is not. Exactly $n$ turns will pass before the \tm ``activates'' and begins the computation. On each of those turns turn, the only option available to each player is to pick an integer that will be encoded as $y_i$ to a tape which will be read by the \tm as an input. After $n$ turns go by, the setup from the original \tm will be in place and neither player will be able to further interfere with the result of the game. This will establish a reduction from the mate-in-$n$ problem to determining the truth of a $\Delta_n^0$ sentence.

The key to this construction is the ability \textit{suspend}. When a card with suspend is played, counters known as suspend counters are place on it. Each turn a suspend counter is removed, and when the last counter is removed the spell is actually cast. When setting up the construction, we can give any spells we wish suspend with \textbf{Delay} (``Counter target spell. If the spell is countered this way, exile it with three time counters on it instead of putting it into its owner's graveyard. If it doesn't have suspend, it gains suspend.'') and can use \textbf{Clockspinning} (``Choose a counter on target permanent or suspended card. Remove that counter from that permanent or card or put another of those counters on it.'') to manipulate the number of suspend counters on any card. Suspend allows us to force the players to find a winning strategy in only $n$ moves by setting it up so that after $n$ turns spells resolve that activate the Turing machine.

Many of the cards necessary for this construction would interfere with the operation of the \textit{Magic} Turing machine. To prevent this, when we activate the \tm we will also have to remove many of the permanents from the battlefield. Most of this clean-up will be accomplished by \textbf{Tetzimoc, Primal Death} (``When Tetzimoc enters the battlefield, destroy each creature your opponents control with a prey counter on it'').

With the exception of \textbf{Tetzimoc, Primal Death}, every creature introduced in this section will actually be a token that is a copy of the card in question, rather than the card itself. Making them tokens means that when they are destroyed they are removed fromt the game instead of going to the graveyard, so \textbf{Wheel of Sun and Moon} doesn't put them on the bottom of Alice's library. Additionally, every one of those tokens that isn't part of the Turing tape will have a prey counter on it, to allow \textbf{Tetzimoc, Primal Death} to destroy them. We will also need some Auras enchanting creatures, which will also be tokens. These Auras are destroyed when the creature they are attached to is, so there is no need to clean them up separately. 

\textbf{Tetzimoc, Primal Death} will be in play under Alice's control and hacked to be a Human, Alice will own a \textbf{Human Frailty} (``Destroy target Human creature'') in exile with $n$ time counters on it, and Bob will control a \textbf{Grave Betrayal} (``Whenever a creature you don't control dies, return it to the battlefield under your control with an additional +1/+1 counter on it at the beginning of the next end step. That creature is a black Zombie in addition to its other colors and types.''). When the last time counter is removed from \textbf{Human Frailty}, Alice is forced to cast it and target \textbf{Tetzimoc, Primal Death}, as there will be no other Humans on the battlefield. This will cause \textbf{Tetzimoc, Primal Death} to die and return to the battlefield under Bob's control, triggering its ability and destroying every creature under Alice's control with a prey counter and every Aura enchanting those creatures. As they are tokens, they cease to exist and are not returned by \textbf{Grave Betrayal}.

In addition to destroying permanents, we will also need to begin the Turing machine itself. To prevent the Turing machine from operating, we give Alice a \textbf{Maralen of the Mornsong} (``Players can't draw cards. At the beginning of each player's draw step, that player loses 3 life, searches their library for a card, puts it into their hand, then shuffles their library.'') and a \textbf{Timelock Orb} (``Players can't search libraries''). Together, these allow us to keep Alice's hand empty until \textbf{Tetzimoc, Primal Death} destroys \textbf{Maralen of the Mornsong}. At that point she will draw her first card (\textbf{Infest}) which she will cast on her next turn to begin the computation. Although \textbf{Timelock Orb} doesn't prevent the loss of life from \textbf{Maralen of the Mornsong}'s ability, the life loss is irrelevant thanks to the life that the \textbf{Daggerdrome Imp} is gaining.

Instead of using \textbf{Blazing Archeon} (``Flying. Creatures can't attack you.'') as specified in the \tm we will use \textbf{Moat} (''Creatures without flying can't attack.'') to allow our \textbf{Daggerdrome Imp} to attack. Additionally, \textbf{Choak} (``Islands don't untap during their controllers' untap steps'') will be suspended with $n$ time counters on it\footnote{As discussed later, we will have Islands we need to tap for mana every turn.}. All other cards from the \tm will be in play as specified by the \tm for our construction. Additionally, our tape will begin pre-initialised to encode a program that reads the player-specified inputs and then searches for a solution to $(p(x,c_1,\ldots,c_n) = 0$, where $c_i$ are the values chosen in the $ith$ round.

In the \tm the authors provide a two-sided infinite tape. This is not strictly speaking necessary, as the (2,18) Turing machine only requires a one-sided infinite tape. As a result, we can use the one side of the tape (say, the left side) to encode the input we wish to have the players supply. This input will be encoded in ``inverse unary'' notation. To denote the natural number $n$ there will be $n$-many blank cells in a row, with a non-blank cell marking the division between numbers. We choose to use the \textit{elemental} creature type to mark the divider symbol.

To create these tokens, Alice will control an \textbf{Ageless Entity} (``Whenever you gain life, put that many +1/+1 counters on Ageless Entity'') equipped with a \textbf{Helm of the Host} (``At the beginning of combat on your turn, create a token that's a copy of equipped creature, except the token isn't legendary if equipped creature is legendary''). She will also control a \textbf{Daggerdrome Imp} (``Flying. Lifelink.'') that has a +1/+1 counter on it and is enchanted by a \textbf{Shade's Form} (`Enchanted creature has ``B: This creature gets +1/+1 until end of turn.'') and a \textbf{Cloak of Mists} (``Enchanted creature cannot be blocked'') and a \textbf{Hellraiser Goblin} (``Creatures you control have haste and attack each combat if able''). Additionally, a \textbf{Pithing Needle} (``As Pithing Needle enters the battlefield, choose a card name. Activated abilities of sources with the chosen name can't be activated unless they're mana abilities.'') will be in play naming \textbf{Helm of the Host}. This prevents Alice from changing which creature it is attached to, but doesn't prevent its copying ability from triggering.

Activating the \textbf{Daggerdrome Imp}'s ability requires Alice to spend black mana. She can have access to infinite amounts of black mana by the combination of \textbf{Umbral Mantle} (``Equipped creature has ``3, Q: This creature gets +2/+2 until end of turn.'' (Q is the untap symbol.)''), \textbf{Magus of the Coffers} (``2, T: Add B for each Swamp you control.''), and six \textbf{Swamp}s. The \textbf{Swamp}s will really be tokens copies of \textbf{Ancient Tomb}, but they will count as swamps due to \textbf{Prismatic Omen} (``Lands you control are every basic land type in addition to their other types.''). They will also be hacked to be creatures, and will have a prey counter on each of them. This \textbf{Prismatic Omen} is why we have to suspend \textbf{Choke}.

Together, these cards mean that Alice's combat steps will go as follows:
\begin{enumerate}
    \item \textbf{Helm of the Host} will create a token that is a copy of \textbf{Ageless Entity}. That token will be a 3/3 due to the \textbf{Night of Souls' Betrayal} from the \tm.
    
    \item \textbf{Hellraiser Goblin} requires Alice to attack with every creature that can. The only one that can is \textbf{Daggerdrome Imp} due to \textbf{Moat} from the \tm. \textbf{Daggerdrome Imp} is unblockable due to \textbf{Cloak of Mists}.
    
    \item \textbf{Daggerdrome Imp} is naturally a 1/1 due to the +1/+1 counter and \textbf{Night of Soul's Betrayal} cancelling each other out. However, Alice can activate \textbf{Daggerdrome Imp}'s ability to increase its power and toughness by 1 however many times she likes.
    
    \item \textbf{Daggerdrome Imp} has lifelink, so Alice will gain life equal to however much damage it deals to Bob. This will trigger every \textbf{Ageless Entity} in play, giving them +1/+1 counters equal to the amount of damage \textbf{Daggerdrome Imp} deals.
\end{enumerate}

This allows Alice to implement the previously described encoding of natural numbers onto the Turing tape, encoding one number per turn. The number of blank cells is equal to the number she chooses to encode, and the next turn an \textbf{Ageless Entity} token is created, forming the separation marker between strings of blank cells.

To allow Bob to write to the tape as well as Alice, we will give Bob the ability to control every other turn Alice takes. This is achieved by giving Bob a \textbf{Panoptic Mirror} (``Imprint — X, T: You may exile an instant or sorcery card with converted mana cost X from your hand.
At the beginning of your upkeep, you may copy a card exiled with Panoptic Mirror. If you do, you may cast the copy without paying its mana cost.'') which has exiled a \textbf{Cruel Entertainment} (``Choose target player and another target player. The first player controls the second player during the second player's next turn, and the second player controls the first player during the first player's next turn''). \textbf{Panoptic Mirror} will also be given phasing with \textbf{Teferi's Curse}.

Each round of turns will correspond to one quantifier. On the $\exists$-rounds, Alice will control her turn and Bob will control his. On Alice's turn, Alice will pick a number to input into the \tm and set it by using \textbf{Daggerdrome Imp}'s ability that many times. On Bob's turn, \textbf{Panoptic Mirror} will be phased in and he will opt to cast \textbf{Cruel Entertainment} with it. This is always in his best interest, as not doing so amounts to forfeiting one of the numbers he gets to choose to Alice. On the $\forall$-rounds, Bob will control Alice's turn and Alice will control Bob's. On Alice's turn, Bob will pick a number to input into the \tm and set it by using \textbf{Daggerdrome Imp}'s ability that many times. On Bob's turn, \textbf{Panoptic Mirror} will be phased out and Alice will have no decisions to make.

The only cards in this construction that cannot be cleaned up with \textbf{Tetzimoc, Primal Death} are \textbf{Panoptic Mirror}, \textbf{Tetzimoc, Primal Death}, and \textbf{Grave Betrayal} as they are all on Bob's side of the battlefield. \textbf{Panoptic Mirror} doesn't actually have to be removed at all, as trading who controls who for each turn is irrelevant once neither side has any strategic decisions to make for the rest of the game. \textbf{Tetzimoc, Primal Death} can be exiled (after it dies once to \textbf{Human Frailty} and is reanimated by \textbf{Grave Betrayal}) by using a \textbf{Ghostflame Sliver} hacked to make Dinosaurs colorless, and having an \textbf{Infernal Reckoning} (``Exile target colorless creature. You gain life equal to its power'') suspended by either player. As there are no other colorless creatures in the construction, it must target \textbf{Tetzimoc, Primal Death}. \textbf{Grave Betrayal} can be handled by enchanting it with \textbf{Reality Acid} (``Enchant permanent. Vanishing 3 (This Aura enters the battlefield with three time counters on it. At the beginning of your upkeep, remove a time counter from it. When the last is removed, sacrifice it.) When Reality Acid leaves the battlefield, enchanted permanent’s controller sacrifices it.''). During the set-up, we can add vanishing counters to \textbf{Reality Acid} so that it runs out of counters right before the \tm begins.
\end{proof}

It is an interesting question whether this problem is in fact $\Delta^0_n$-complete. While it's not $\Delta^0_n$-complete when the ``mate-in-$n$'' game is defined in terms of the traditional notion, we may be hopeful that it is $\Delta^0_n$-complete when defined in terms of \textit{end games}. Finally, we present our third theorem:

\twop*
\begin{proof}
In the proof of \autoref{thm:mate} we give, for every $n$, constructions of board states that are $\Delta^0_n$-hard. Taking the union of those sets of board states gives a set for which identifying the winning strategy is $\Delta^0_n$-hard for every $n$, and so is non-arithmetic in general.
\end{proof}
It seems highly plausible that the correct upper bound on the computational complexity of identifying a winning strategy in \textit{Magic: the Gathering} is precisely $\emptyset^{(\omega)}$. For that to not be the case, it would have to be that our proof of \textbf{Theorem \ref{thm:2player}} is misleading in the sense that the primary complexity is not due to the fact that games of \textit{Magic} can have arbitrary length. We can make rigorous this idea as follows: if the mate-in-$n$ problem is equivalent to $\emptyset^{(f(n))}$ for some monotonic function $f:\mathbb N\to\mathbb N$ then it is the case that the general strategy is precisely equivalent to $\emptyset^{(\omega)}$.

If the general strategy is harder than $\emptyset^{(\omega)}$ then non-arithmeticity must arise in a finite number of turns. Although another construction could potentially fit arbitrarily many choices into finitely many turns, the combinatorial structure of game actions in \textit{Magic: the Gathering} means that this doesn't avoid the issue presented in the previous paragraph, as the non-arithmeticity must arise in a finite number of \textit{game actions} as well. Although this is implausible, it's unclear how to prove it. One approach would be to encode the entirety of the game in the language of arithmetic so that there is a reduction of winning strategies to the truth of arithmetic expressions. Although the axiomatic nature of the rules of \textit{Magic} makes this more straightforward than other games, it would still be a colossal undertaking. Also, as new cards get printed it is possible that the proof will be invalidated by the new cards. Therefore we leave as a conjecture:
\begin{conj}
Determining if there exists a winning strategy in \textit{Magic: the Gathering} is Turing equivalent to $\emptyset^{(\omega)}$.
\end{conj}

One plausible route for falsifying this conjecture is to construct a second-order syntax in \textit{Magic: the Gathering}. Some cards in \textit{Magic} have quite complicated effects, including ones that require you to choose between subsets of a set. It is plausible that some of these cards are able to express second order arithmetical statements, in which case we would expect the ideas developed in this paper to generalise to constructions in the Borel hierarchy.

\section{Playability in the Real World}\label{sec:real-world}

One thing that distinguishes \textit{Magic: the Gathering} from other games is that it's complexity is actually found in the real world. While most games are analysed based on some sort of generalisation, everything discussed in this paper is achievable in a real game of \textit{Magic}. The following 60-card deck is legal to play in the Legacy format and allows a sufficiently tenacious player to set up a board state for which optimal play requires knowing the truth of an arithmetic statement of the player's choice:

\begin{table*}[htbp]  
    \centering
    \caption{60-Card deck list to play the Turing machine in a Legacy tournament}
    \label{tab:decklist}
    \begin{tabular}{l|l|l|l}
    Device Set Up & Device Set Up & \tm & Arithmetic Sentences\\
    \hline
    1 Ancient Tomb & 1 Memnarch & 1 Rotlung Reanimator & 1 Tetzimoc, Primal Death\\
    1 Grim Monolith & 1 Artificial Evolution & 1 Infest & 1 Grave Betrayal\\
    1 Power Artifact & 1 Dread of Night & 1 Cleansing Beam & 1 Maralen of the Mornsong\\
    1 Gemstone Array & 1 Glamerdye & 1 Soul Snuffers & 1 Timelock Orb\\
    1 Staff of Domination & 1 Prismatic Lace & 1 Illusionary Gains & 1 Ageless Entity\\
    1 Karn Liberated & 1 Donate & 1 Priviledged Position & 1 Helm of the Host\\
    1 Fathom Feeder & 1 Reality Ripple & 1 Steely Resolve & 1 Daggerdrome Imp\\
    1 Cloak of Mists & 1 Riptide Replicator & 1 Wild Evocation & 1 Umbral Mantle\\
    3 Lotus Petal & 1 Stolen Identity & 1 Shared Triumph & 1 Hellraiser Goblin\\
    1 Ghostflame Sliver & 1 Capsize & 1 Xanthrid Necromancer & 1 Magus of the Coffers\\
    1 Infernal Reckoning & 1 Clockspinning & 1 Mesmeric Orb & 1 Moat\\
    1 Reality Acid & 1 Delay & 1 Coalition Victory & 1 Cruel Entertainment\\
    1 Cloak of Invisibility & 1 Wheel of Sun and Moon & 1 Choke & 1 Panoptic Mirror\\
    1 Rings of Brighthearth & 1 Teferi's Curse & 1 Vigor & 1 Pithing Needle\\
    & 1 Fungus Sliver & 1 Prismatic Omen & \\
   \end{tabular}
\end{table*}

As raised by Churchill, Biderman, and Herrick in their paper \cite{cbh:mtg-tm}, the rules regarding slow play are a potential hurdle for carrying out this construction in practice. While carrying out this construction in a tournament setting would probably result in the player getting punished for deliberately delaying the game (to be honest, there isn't much else of a reason to build this), our understanding of the rules is that a player who happens to find themselves in such a position or a player who finds themselves in a less contrived but mathematically equivalent situation would not be sanctioned for slow play. 

\section{Conclusions and Future Work}\label{futurework}

We have established that optimal play in \textit{Magic: the Gathering} is at least as hard as $\emptyset^{(\omega)}$, the first time any real or even realistic game has been shown to be this complex. We have identified several interesting computability theoretic open questions about \textit{Magic} that remain, most notably ``is our result optimal?'' Another avenue for future research is to examine other collectable card games, such as Yu-Gi-Oh and Pokemon TCG and see if similar results can be proven for those games.

\section*{Acknowledgements}

We would like to thank Alex Churchill, Austin Herrick, and Nicholas Wild for conversations about the mechanisms underlying this paper and consultation about complex rules interactions in \textit{Magic: the Gathering}.

\bibliographystyle{plain}
\bibliography{fun_submission}

\begin{thebibliography}{10}

\bibitem{ci:blocking}
Krishnendu Chatterjee and Rasmus Ibsen-Jensen.
\newblock The complexity of deciding legality of a single step of
  \textrm{Magic: The Gathering}.
\newblock In {\em 22nd European Conference on Artificial Intelligence}, 2016.

\bibitem{c:old-tm}
Alex Churchill.
\newblock \textit{Magic: The Gathering} is \textrm{Turing} complete v5, 2012.
\newblock \url{https://www.toothycat.net/~hologram/Turing/}.

\bibitem{cbh:mtg-tm}
Alex Churchill, Stella Biderman, and Austin Herrick.
\newblock \textit{Magic: The Gathering} is turing complete, 2019.

\bibitem{c:forum}
Alex Churchill et~al.
\newblock \textrm{Magic is Turing} complete (the \textrm{Turing} machine
  combo), 2014.
\newblock \url{http://tinyurl.com/pv3n2lg}.

\bibitem{cl:smash}
Michael~J. Coulombe and Jayson Lynch.
\newblock Cooperating in video games? \textrm{Impossible! Undecidability} of
  team multiplayer games.
\newblock In {\em 9th International Conference on Fun with Algorithms}, 2018.

\bibitem{dkl:recursed-not-recursive}
Erik Demaine, Justin Kopinsky, and Jayson Lynch.
\newblock Recursed is not recursive: A jarring result, 2020.

\bibitem{dh:contraint-logic}
Erik~D. Demaine and Robert~A. Hearn.
\newblock Constraint logic: A uniform framework for modeling computation as
  games.
\newblock In {\em 2008 23rd Annual IEEE Conference on Computational
  Complexity}, pages 149--162, 2008.

\bibitem{e:thesis}
Alexander Esche.
\newblock {\em Mathematical Programming and \textrm{Magic: The Gathering}}.
\newblock PhD thesis, Northern Illinois University, 2018.

\bibitem{h:braid-undecidable}
Linus Hamilton.
\newblock Braid is undecidable, 2014.

\bibitem{wc:mcmc}
Colin~D. Ward and Peter~I. Cowling.
\newblock \textrm{Monte Carlo} search applied to card selection in
  \textrm{Magic: The Gathering}.
\newblock In {\em CIG'09 Proceedings of the 5th international conference on
  Computational Intelligence and Games}, pages 9--16, 2009.

\bibitem{wcp:mcmc}
Colin~D. Ward, Peter~I. Cowling, and Edward~J. Powley.
\newblock Ensemble determinization in \textrm{Monte Carlo} tree search for the
  imperfect information card game \textrm{Magic: The Gathering}.
\newblock In {\em IEEE Transactions on Computational Intelligence and AI in
  Games}, volume~4, 2012.

\bibitem{rules}
{Wizards of the Coast}.
\newblock \textrm{Magic: The Gathering} comprehensive rules, Aug 2018.
\newblock
  \url{https://magic.wizards.com/en/game-info/gameplay/rules-and-formats/rules}.

\end{thebibliography}

\begin{appendices}
\section{How to Play \textit{Magic: the Gathering}}\label{app:rules}
In this appendix we provide a brief overview of the game and its rules, with a focus on what is necessary to understand the Turing machine construction. The full \textit{Magic: the Gathering} Comprehensive Rules document \cite{rules} is over 200 pages of text and detailing them falls outside the purview of this paper.

\subsection{An Introduction to \textit{Magic}}

\textit{Magic: the Gathering} is a card game about magical combat. Each player begins with a deck of cards that they've chosen called their \textit{library}. Game proceeds by drawing cards from the library, casting spells from their hand to summon creatures or create effects, and attacking with their creatures. Creatures can engage in combat and deal damage to the opponent's creatures, as well as to the opponent themselves. When creatures die or when one-time effects are used, those cards are placed in a discard pile called the \textit{graveyard}. Each player begins with 20 \textit{life points} and once they are depleted that player loses the game. There are a few auxiliary ways that a game of \textit{Magic} can end, but they will not be relevant to our construction beyond ensuring that they do not occur.
\subsection{Types of Cards}
One important attribute of cards in \textit{Magic: the Gathering} is the \textit{type}. Cards with different types are affected by different cards and have different rules associated with them. There are five types that we will use in our construction. Each \textbf{bolded} term in this list is a type. 
\begin{enumerate}
    \item \textbf{Creatures} are permanents, which means that they stay in play after they've been cast. 
    Creatures are the only type of card that can engage in combat directly. Creatures have \textit{power} and \textit{toughness}, which determine their strength in combat: power determines how much damage they do, and toughness determines how much life they have. Standard notation for describing power and toughness separates them with a slash: a 3/2 creature has 3 power and 2 toughness. Creatures need to have a non-zero toughness to remain on the battlefield: any time a creature's toughness becomes 0 or less for any reason, it dies and is sent to the graveyard. Creatures are the primary component of the Turing machine and comprise both the tape and the head.
    \item \textbf{Artifacts} and \textbf{Enchantments} are also permanents, but do not have power/toughness and do not have the ability to attack. There are several artifacts and enchantments in our construction that provide important effects to keep the Turing machine running. Some enchantments are \textit{Aura}s and some artifacts are \textit{Equipment}: these two card types can be attached to one other permanent or player and modify or affect that permanent or player in some way. Artifacts and enchantments are two separate card types, though the difference between them is never relevant for our construction.
    \item \textbf{Instants} and \textbf{Sorceries} are cards that generate one-time effects, and are immediately discarded after being used, as opposed to being left in play the way permanents are.
    \item \textbf{Lands} are permanents that do not have a direct influence on the game. Instead, they provide a resource known as \textit{mana} which is required to cast most spells and activate most abilities.
\end{enumerate}
Some cards have subtypes in addition to types. Aura is an example of a subtype of enchantments. All creatures have subtypes called \textit{creature types} such as Goblin or Wizard that denote their race or class, and those creature types are used in various ways throughout the construction, in particular to track the symbols written onto the Turing tape.

\subsection{Editing Card Text and Types}
The Turing machine construction is only possible because certain \textit{Magic} cards allow modification of the text of other cards, to change colours or creature types. 
The card \textbf{Artificial Evolution} reads ``Change the text of target spell or permanent by replacing all instances of one creature type with another. The new creature type can't be Wall. \textit{(This effect lasts indefinitely.)}''

Also crucial is one of several cards such as \textbf{Glamerdye} which read ``Change the text of target spell or permanent by replacing all instances of one colour word with another''. Similarly, we can change what colour a permanent is with \textbf{Prismatic Lace} (``Target permanent becomes the colour or colours of your choice'').

For example, the card \textbf{Rotlung Reanimator} reads ``Whenever Rotlung Reanimator or another Cleric dies, create a 2/2 black Zombie creature token''.  By casting two copies of \textbf{Artificial Evolution} replacing `Cleric' with `Aetherborn' and `Zombie' with `Sliver', and one copy of \textbf{Glamerdye} to replace `black' with `white', we can change \textbf{Rotlung Reanimator} to read instead ``Whenever Rotlung Reanimator or another \textit{Aetherborn} dies, create a 2/2 \textit{white} \textit{Sliver} creature token''. This allows us to use creature types to track values throughout the computation, killing creature tokens with particular types and using \textbf{Rotlung Reanimator} as a conditional logic gate.

\textbf{Artificial Evolution} can be used to modify a creature's type as well as its text. It is useful to add extra creature types to some creatures without changing their text box: this can be accomplished with \textbf{Olivia Voldaren} (who has the ability ``Olivia Voldaren deals 1 damage to another target creature. That creature becomes a Vampire in addition to its other types''). We use \textbf{Artificial Evolution} to change \textbf{Olivia Voldaren} to add the creature type `Assembly-Worker' instead of `Vampire': we will use the type Assembly-Worker to denote infrastructure creatures which must be protected from all damage.

It should be noted that all these edits only persist for as long as the permanent remains on the battlefield. If an edited permanent changes zone, such as going to the graveyard or the library, these edits are lost.

\subsection{Tokens}
Some effects can create \textit{tokens} on the battlefield, which are also permanents. This is crucial to the construction of a Turing tape potentially millions of cells long with a bounded number of cards. Tokens may be creatures, generally with no abilities, or they may be copies of other permanents such as enchantments or artifacts. Unless an effect specifies otherwise, tokens are treated exactly like cards of the same type while they are on the battlefield.

Tokens can only exist on the battlefield --- if they ever leave the battlefield they cease to exist. If a creature token is dealt lethal damage, it dies, leaves the battlefield, and goes to the graveyard (triggering any effects that watch for those conditions such as \textbf{Rotlung Reanimator}'s). However, it does not continue to exist in the graveyard.  

\subsection{Abilities and the Stack}
There are many different types of abilities that cards in \textit{Magic: the Gathering} can have. The rules surrounding using abilities get rather complicated, but are crucial to understanding the mechanisms of the constructions in this paper. In this section, we restrict ourselves to explaining the bare minimum required to understand the construction.

Our construction is primarily concerned with \textit{static abilities} and \textit{triggered abilities}. Static abilities are abilities that are ``always on'' and modify the general rules of the game. For example, \textbf{Lure} reads ``Enchant creature. All creatures able to block enchanted creature do so.'' This is a static ability of the \textbf{Lure} permanent, affecting a creature, and removing all player choices about how to block that creature. \textbf{Lure} itself is an enchantment with subtype Aura.

Triggered abilities begin with one of the words ``When'', ``Whenever'' or ``At''. \textbf{Rotlung Reanimator} has a triggered ability that reads ``Whenever Rotlung Reanimator or another Cleric dies, create a 2/2 black Zombie creature token.'' \textbf{Rotlung Reanimator} is how we will perform many of the functions of the Turing machine. It is a creature with two subtypes: Zombie and Cleric.

Whenever a spell is cast or an ability is activated or triggered, it is first put in a holding area known as the \textit{stack}. When a spell or ability is on the stack, other players may add additional spells or abilities to the stack before the effect \textit{resolves} (takes effect). The stack in \textit{Magic} functions exactly like the data structure of the same name, with the spell or ability put on the stack first being carried out last and the spell or ability put on the stack last being carried out first. Players take turns getting \textit{priority}, which is the game's permission to cast spells and activate abilities. The player whose turn it is always gets priority first, and then the player whose turn it isn't. Once both players decide to not use their priority to put a spell or ability on the stack, the top effect on the stack is popped and resolves.

Sometimes two triggered abilities will try to go on the stack at the same time. In this case, the order is determined by \textit{Active Player, Nonactive Player (APNAP)} order. The active player is the one whose turn it is. Since this is the order the spells and abilities go on the stack, they will resolve in the reverse order (so the nonactive player's ability resolves first). If both effects are controlled by the same player, that player must choose the order to place them onto the stack.

\subsection{Phasing}\label{phasing}
\textit{Phasing} is an unusual ability some \textit{Magic: the Gathering} cards have that is crucial to the \tm. It allows a creature to be treated as if it doesn't exist -- in particular, its triggered abilities won't trigger -- but it stays on the battlefield, and so edits to its text by \textbf{Artificial Evolution} remain. At the very beginning of a player's turn (their untap step), all their phased-in permanents with phasing `phase out' (temporarily cease to exist) and all their phased-out permanents `phase in' (come back into existence).

\subsection{Counters, Combat, and Damage}
There are many effects that can change the power and toughness of creatures. Some of these are temporary and last until the end of turn, while others are permanent. Permanent changes are denoted by \textit{counters} placed on the creatures. In our construction, we will utilise +1/+1 and \minusminus counters. +1/+1 counters increase the power and toughness of a creature each by $1$, while \minusminus counters decrease them. If a creature ever has both +1/+1 and \minusminus counters, they cancel out and pairs of counters are removed until that creature only has one type of counters or has no counters.

Creatures can engage in combat. In order to do so, the attacking player (you may only attack on your turn) chooses the creatures they wish to attack with and announces that choice. Then the defending player may opt to have some of their creatures ``block.'' A creature chosen to block can only block one attacking creature, though multiple creatures can block the same attacking creature. Once blockers are chosen, all creatures deal damage simultaneously. Creatures deal damage based on their power, so a $3$ power creature deals $3$ damage.

Normally, damage dealt to creatures doesn't do anything unless the cumulative amount of damage dealt to a creature over the course of a turn equals or exceeds its toughness, at which point that creature dies and is sent to the graveyard. At the end of the turn, creatures are reset to having zero damage. For our construction, however, it is more convenient if most damage persists through several turns, so we utilize an ability known as \textit{infect}. Creatures with infect deal damage to creatures by putting a number of \minusminus counters onto the other creature equal to the amount of damage that they would have dealt. This is done instead of dealing normal damage.

When damage is dealt to players, that player's life total is decreased by the amount of damage dealt (players start with $20$ life, and a player with zero or fewer life loses the game). If the source of that damage has infect, the damage is dealt by giving the player ``poison counters'' instead. A player with $10$ or more poison counters loses the game.

Players losing the game, creatures dying due to having taken too much damage, and cancelling +1/+1 and \minusminus counters are known as \textit{state-based actions}. State-based actions are all ``cleanup'' activities that maintain the correct state of the game. Every time a player would gain priority or an effect finishes resolving, players check if there are any state-based actions that need to be carried out, and perform them if so. All relevant state-based actions occur simultaneously, do not use the stack, and cannot be responded to.

\subsection{The Structure of a Turn}
Play in \textit{Magic: the Gathering} consists of players alternatively taking turns. Each turn is divided into phases, with each phase divided into steps such as the \textit{upkeep step}. Many cards in \textit{Magic} say something like ``At the beginning of your upkeep...'' or ``At the beginning of your main phase...'' At the beginning of each step and phase, the first thing done is always to check for such abilities and put them on the stack. During each of these phases, there is the option to cast spells and activate abilities, but some additionally have game actions players are required to take after all relevant effects have resolved.

The first phase of each turn is the \textit{beginning phase}, which consists of the \textit{untap step}, the \textit{upkeep step}, and the \textit{draw step}. During the untap step players first carry out any phasing effects, and then untap all permanents they control. There are no game actions during the upkeep step, and during the draw step the active player draws one card.

The second phase is the \textit{first main phase}, where the bulk of the play occurs during a normal game, though nothing happens in our construction. The third phase is the \textit{combat phase}, which is where combat occurs.  We will also have nothing relevant happen in the fourth phase (the \textit{second main phase}) and only minimal effects during the final \textit{end phase}.
\end{appendices}

\end{document}